\documentclass{article}

     \usepackage[preprint, nonatbib]{neurips_2020}

\usepackage[utf8]{inputenc} 
\usepackage[T1]{fontenc}    
\usepackage{hyperref}       
\usepackage{url}            
\usepackage{booktabs}       
\usepackage{amsfonts}       
\usepackage{nicefrac}       
\usepackage{microtype}      
\usepackage{amsmath, amssymb, amsthm, mathtools, esint}

\theoremstyle{plain}
\begingroup
\newtheorem{theorem}{Theorem}
\newtheorem*{theorem*}{Theorem}
\newtheorem*{"theorem"}{``Theorem''}

\newtheorem{lemma}[theorem]{Lemma}
\endgroup

\theoremstyle{definition}
\begingroup

\endgroup

\theoremstyle{remark}
\begingroup
\newtheorem{remark}[theorem]{Remark}

\endgroup

\renewcommand{\d}{\mathrm{d}}
\newcommand{\dx}{\,\mathrm{d}x}

\newcommand{\R}{\mathbb R} 
\newcommand{\N}{\mathbb N}

\newcommand{\spt}{{\mathrm{spt}}}
\newcommand{\cc}{\Subset}
\newcommand{\B}{{\mathcal B}}
\newcommand{\F}{{\mathcal F}}

\allowdisplaybreaks

\newcommand{\Risk}{\mathcal{R}}
\renewcommand{\P}{\mathbb{P}}
\newcommand{\ReLU}{\mathrm{ReLU}}
\newenvironment{pde}{\left\{\begin{array}{rll} } {\end{array}\right.}

\title{Can Shallow Neural Networks Beat the Curse of Dimensionality?\\{ A mean field training perspective}}

\author{%
Stephan Wojtowytsch\\
PACM\\
Princeton University\\
Princeton, NJ 08544\\
  \texttt{stephanw@princeton.edu} \\
  \And
  Weinan E\\
Department of Mathematics and PACM\\
 Princeton University\\ 
 Princeton, NJ 08544\\ 
   \texttt{weinan@math.princeton.edu} \\
}

\begin{document}

\maketitle

\begin{abstract}
We prove that the gradient descent training of a two-layer neural network on empirical or population risk may not decrease population risk at an order faster than $t^{-4/(d-2)}$ under mean field scaling.  
Thus gradient descent training for fitting reasonably smooth, but truly high-dimensional data may be subject to the curse of dimensionality. We present numerical evidence that gradient descent training with general Lipschitz target functions becomes slower and slower as the dimension increases, but converges at approximately the same rate in all dimensions when the target function lies in the natural function space for two-layer ReLU networks.
%
\end{abstract}

\section{Introduction}

Since Barron's seminal article \cite{barron1993universal}, artificial neural networks have been celebrated as a tool to beat the curse of dimensionality. Barron proved that two-layer neural networks with $m$ neurons and suitable non-linear activation can approximate a large (infinite-dimensional) class of functions $X$ to within an error of order $1/\sqrt m$ in $L^2(\mathbb P)$ for any Radon probability measure $\mathbb P$ on $[0,1]^d$ {\em independently of dimension $d$}, while any sequence of linear function spaces $V_m$ with $\dim(V_m) = m$ suffers from the curse of dimensionality if the data distribution $\mathbb P$ is truly high-dimensional. More specifically
\[
\sup_{\|\phi\|_X\leq 1} \inf_{\psi\in V_m} \|\phi - \psi\|_{L^2(\P)} \geq \frac{c}{d} \,m^{-1/d}
\]
for a universal constant $c>0$ if $\P$ is the uniform measure on $[0,1]^d$ and $X$ describes the same function class that is approximated well by neural networks with $O(m)$ parameters and $\|\cdot\|_X$ denotes its natural norm. Thus from the perspective of approximation theory, neural networks leave linear approximation in the dust in high dimensions.

The perspective of approximation theory only establishes the existence of neural networks which approximate a given target function well in some sense, while in applications, it is important to find optimal (or at least reasonably good) parameter values for the network. The most common approach is to initialize the parameters randomly and optimize them by a gradient-descent based method. We focus on the case where the goal is to approximate a target function $f^*$ in $L^2(\P)$ for some Radon probability measure $\P$ on $[0,1]^d$. To optimize the parameters $\Theta = \{(a_i, w_i, b_i)\}_{i=1}^m$ of two-layer network
\[
f_\Theta(x) = \sum_{i=1}^m a_i\,\sigma(w_i^Tx + b_i), 
\]
we therefore let $\Theta$ evolve by the gradient flow of the risk functional
 \[
\Risk(\Theta) = \frac12 \int_{[0,1]^d} \big(f_\Theta(x) - f^*(x)\big)^2\,\P(\dx).
\]
In practice, we only have access to data sampled from an unknown underlying distribution $\mathbb P$. The approximation therefore takes place in $L^2(\mathbb P_n)$ instead of $L^2(\P)$ where
$\mathbb P_n = \frac1n \sum_{j=1}^n \delta_{x_j}
$ is the empirical measure of the data samples. If all data points are sampled independently, the empirical measures converge to the underlying distribution $\mathbb P$. In this article, we focus on uniform estimates in the number of data samples and population risk.

While the optimization problem is non-convex, gradient flow-based optimization works astonishingly well in applications. The mechanism behind this is not fully understood. In certain scaling regimes in the number of parameters $m$ and the number of data points $n$, the empirical risk has been shown to decay exponentially (with high probability over the initialization), even when the target function values $y_j:= f^*(x_j)$ are chosen randomly in a bounded interval \cite{du2018gradient,weinan2019comparative}. 

Networks which easily fit random data can be trusted to have questionable generalization properties. Even at initialization, network parameters are often chosen too large to retain reasonable control of the {\em path norm}, which controls the generalization error. This allows the network to fit any data sample with minimal change in the parameters, behaving much like its linearization around the initial configuration (an infinitely wide random feature model), see \cite{weinan2019comparative}. This approach explains how very wide two-layer networks behave, but  it 
does not explain why neural networks are more powerful in applications than random feature models.

On the opposite side of the spectrum lies the mean field regime \cite{chizat2018global,mei2018mean, rotskoff2018neural,sirignano2018mean}.  
Under mean field scaling two-layer network with $m$ neurons and parameters $\Theta = \{(a_i, w_i, b_i) \in \R\times\R^d\times\R\}_{i=1}^m$ is given as
\[
f_\Theta(x) = \frac1m\sum_{i=1}^m a_i\,\sigma(w_i^T x + b_i)\quad\text{rather than } f_\Theta(x) = \sum_{i=1}^m a_i\,\sigma(w_i^T x + b_i).
\]
Both concepts of neural network are equivalent from the perspective of approximation theory (static), but behave entirely differently under gradient descent training (dynamics), see e.g.\ \cite{chizat2018note}. In the mean field regime, parameters may move a significant distance from their initialization, making use of the adaptive feature choice in neural networks compared to random feature models. This regime thus has greater potential to establish the superiority of artificial neural networks over kernel methods.

Mean field gradient flows do not resemble their linearization at the initial condition. The convergence of gradient descent training to minimizers of the often highly non-convex loss functionals is therefore not obvious (and, for poorly chosen initial values, generally not true). Even if empirical and population risk decay to zero along the gradient flow, population risk may do so at very slow rates in high dimension.

\begin{theorem}\label{main theorem}
Let $\sigma$ be a Lipschitz-continuous activation function.
Consider population and empirical risk expressed by the functionals 
\[
\Risk(\Theta) = \frac12 \int_{[0,1]^d} (f_\Theta- f^*)^2(x)\dx, \qquad \Risk_n(\Theta) = \frac1{2n} \sum_{i=1}^n   (f_\Theta- f^*)^2(x_i)
\]
where $f^*$ is a Lipschitz-continuous target function and the points $x_i$ are iid samples from the uniform distribution on $[0,1]^d$. 
There exists $f^*$ with Lipschitz constant and $L^\infty$-norm bounded by $1$ such that parameters $\Theta_t$ evolving by the gradient flow of either $\Risk_n$ or $\Risk$ itself satisfy
$\limsup_{t\to \infty} \big[t^\gamma\,\Risk(\Theta_t)\big] = \infty
$ for all $\gamma> \frac{4}{d-2}$. 
\end{theorem}

Intuitively, this means that the estimate $\Risk(\Theta_t) \geq t^{-\frac{4}{d-2}}$ is almost true. The result holds uniformly in $m$ and even for infinitely wide networks. An infinitely wide mean field two-layer network (or Barron function) is a function
\[
f_\pi(x) = \int_{\R\times\R^d\times \R} a\,\sigma\big(w^Tx+b\big)\,\pi(\d a\otimes \d w\otimes \d b)
\]
where $\pi$ is a suitable Radon probability measure on $\R^{d+2}$. Networks of finite width are included in this definition by setting
$\pi = \frac1m\sum_{i=1}^m \delta_{(a_i, w_i,b_i)}.
$ It has been observed (see e.g.\ \cite[Proposition B.1]{chizat2018global}) that the vectors $(a_i, w_i, b_i)$ move by the usual gradient flow of $\Risk$ if and only if the associated measure $\pi$ evolves by the time-rescaled Wasserstein gradient flow of
\[
\Risk(\pi):= \frac12\int_{[0,1]^d} \big(f_\pi - f^*\big)^2(x)\:\P(\d x).
\]
We show the following more general result which implies Theorem \ref{main theorem}.

\begin{theorem}\label{main theorem 2}
Let $\sigma$ be a Lipschitz-continuous activation function.
Consider population and empirical risk expressed by the functionals 
\[
\Risk(\pi) = \frac12 \int_{[0,1]^d} (f_\pi- f^*)^2(x)\dx, \qquad \Risk_n(\pi) = \frac1{2n} \sum_{i=1}^n   (f_\pi- f^*)^2(x_i)
\]
where $f^*$ is a Lipschitz-continuous target function and the points $x_i$ are iid samples from the uniform distribution on $[0,1]^d$. 
There exists $f^*$ with Lipschitz constant and $L^\infty$-norm bounded by $1$
such that parameter measures $\pi_t$ evolving by the $2$-Wasserstein gradient flow of either $\Risk_n$ or $\Risk$ satisfy
\[
\limsup_{t\to \infty} \big[t^\gamma\,\Risk(\pi_t)\big] = \infty
\]
for all $\gamma> \frac{4}{d-2}$. 
\end{theorem}

Theorem \ref{main theorem 2} provides a more general perspective than Theorem \ref{main theorem}. The Wasserstein gradient flow of $\Risk$ is given by the continuity equation
\[
\dot \pi_t = \operatorname{div} \big(\pi_t\,\nabla(\delta_\pi \Risk)\big)\qquad\text{where } (\delta_\pi \Risk)(a,w,b) = \int_{[0,1]^d} (f_\pi -f^*)(x)\,a\,\sigma(w^Tx+b)\,\P(\d x)
\]
is the variational gradient of the risk functional. In particular, any other discretization of this PDE experiences the same curse of dimensionality phenomenon. Besides gradient descent training, this also captures stochastic gradient descent with large batch size and small time steps (to leading order). Viewing machine learning through the lens of classical numerical analysis may illuminate the large data and many parameter regime, see \cite{E:2019aa}. 


The article is structured as follows. In the remainder of the introduction, we discuss some previous works on related questions. In Section \ref{section wasserstein}, we discuss Wasserstein gradient flows for mean-field two-layer neural networks and review a result from approximation theory. 
Next, we show in Section \ref{section synthesis} that Wasserstein gradient flows for two-layer neural network training may experience a curse of dimensionality phenomenon. The analytical result is backed up by numerical evidence in Section \ref{section simulation}. We conclude the article by discussing the significance of our result and related open problems in Section \ref{section discussion}. In an appendix, we show that a similar phenomenon can be established when training an infinitely wide random feature model on a single neuron target function.

\subsection{Previous Work}

The study of mean field training for neural networks with a single hidden layer has been initiated independently in several works \cite{chizat2018global,rotskoff2018neural,sirignano2018mean,mei2018mean}. In \cite{chizat2018note}, the authors compare mean field and classical training. \cite{chizat2018global, Chizat:2020aa, arbel2019maximum} contain an analysis of whether gradient flows starting at a suitable initial condition converge to their global minimum. This analysis is extended to networks with ReLU activation in \cite{relutraining}.

In \cite{hu2019mean}, the authors consider a training algorithm where standard Gaussian noise is added to the parameter gradient of the risk functional. The evolution of network parameters is described by the Wasserstein gradient flow of an energy functional which combines the loss functional and an entropy regularization. In this case, the parameter distribution approaches the stationary measure of a Markov process as time approaches infinity, which is close to a minimizer of the mean field risk functional if noise is small. Note, however, that these results do not describe the small batch stochastic gradient descent algorithm used in practice, for which noise may be assumed to be Gaussian, but with a complicated parameter-dependent covariance structure \cite{hu2019diffusion,li2015dynamics}.

Some results in \cite{chizat2018global} also apply to deeper structures with more than one hidden layer. However, the imposition of a linear structure implies that each neuron in the outer layer has its own set of parameters for the deeper layers. A mean field training theory for more realistic deep networks has been developed heuristically in \cite{nguyen2019mean} and rigorously in \cite{araujo2019mean,nguyen2020rigorous, sirignano2019mean} under the assumption that the parameters in different layers are initialized independently. The distribution of parameters remains a product measure for positive time, so that cross-interactions with infinitely many particles in the following layer (as width approaches infinity) are replaced by ensemble averages. This `propagation of chaos' is the key ingredient of the analysis.

In \cite{abbe2018provable}, the author takes a different approach to establish limitations of neural network models in machine learning, see also \cite{shamir2018distribution,raz2018fast}. Our approach is different in that we allow networks of infinite width and infinite amounts of data.

\section{Background}\label{section wasserstein}

\subsection{Why Wasserstein?}

Let us quickly summarize the rationale behind studying Wasserstein gradient flows of risk functionals. This section only serves as rough overview, see \cite{Chizat:2020aa} for a more thorough introduction to Wasserstein gradient flows for machine learning and \cite{ambrosio2008gradient,santambrogio2015optimal,villani2008optimal} for Wasserstein gradient flows and optimal transport in general.

Consider a general function class $\mathcal F$ whose elements can be represented as normalized sums 
\[
f_{\{\theta_1,\dots,\theta_m\}}(x) = \frac1m \sum_{i=1}^m \phi(x,\theta_i)\quad \text{or more generally averages}\quad f_\pi(x) = \int_\Theta \phi(\theta, x) \,\pi(\d\theta).
\]
of functions in a parameterized family $\{\phi(\cdot,\theta)\}_{\theta \in \Theta}$.
In the case of two-layer networks, $\theta = (a,w,b)$ and $\phi(\theta,x) = a\,\sigma(w^Tx+b)$. If the activation function is $\sigma(z) = \ReLU(z) = \max\{z,0\}$, then $\phi(\lambda \theta, x) = \lambda^2\phi(\theta,x)$ for all $\lambda>0$. Thus $f_\pi$ is well-defined if $\pi$ has finite second moments, i.e.\ $\pi$ lies in the Wasserstein space $\mathcal P_2$. We consider the risk functional
\[
\Risk(\pi) = \frac12\int_{\R^d} \big(f_\pi - f^*\big)^2(x)\,\mathbb P(\d x)
\]
for some data distribution $\mathbb P$ on $\R^d$. Note that $\inf_\pi \Risk(\pi) = 0$ if $\spt(\mathbb P)$ is compact and the class $\{f_\pi |\pi \in \mathcal P_2\}$ has the uniform approximation property on compact sets  (by which we mean that the class is dense in $C^0(K)$ for all $K\cc\R^d$). This is the case for two-layer networks with non-polynomial activation functions -- see e.g.\ \cite{cybenko1989approximation, hornik1991approximation} for continuous sigmoidal activation functions. The same result holds for ReLU activation since $z\mapsto \mathrm{ReLU}(z+1) - \mathrm{ReLU}(z)$ is sigmoidal.

\begin{lemma}\cite[Proposition B.1]{chizat2018global}\label{lemma why wasserstein}
The parameters $\Theta = (\theta_i)_{i=1}^m$ evolve by the time-accelerated gradient flow
\[
\frac{d}{dt}\theta_i(t) = - m\,\nabla_{\theta_i}\Risk(\Theta_t) = - \int \big(f_\Theta- f^*\big)(x)\,\nabla_\theta\phi(\theta_i, x)\,\P(\d x)
\]
of $\Risk$ if and only if their distribution
$
\pi^m_{t} = \frac1m\sum_{i=1}^m \delta_{\theta_i(t)}
$
evolves by the Wasserstein gradient flow
\[
\dot\pi_t = \operatorname{div}\left(\pi_t \,\nabla_{\theta} \frac{\delta\Risk}{\delta\pi}(\pi_t;\cdot)\right)
\qquad
\text{where}
\quad
\frac{\delta\Risk}{\delta\pi}(\pi; \theta) = \int \big(f_\pi-f^*\big)(x)\,\phi(\theta, x)\,\mathbb P(\d x).
\]
\end{lemma}

The continuity equation describing the gradient flow is understood in the sense of distributions. By the equivalence in Lemma \ref{lemma why wasserstein}, all results below apply to networks with finitely many neurons as well as infinitely wide mean field networks. In this article, we do not concern ourselves with existence for the gradient flow equations. More details can be found in \cite{chizat2018global} for general activation functions with a higher degree of smoothness and in \cite{relutraining} for ReLU activation.

\subsection{Growth of Second Moments}
Denote the second moment of $\pi$ by
$
N(\pi) := \int|\theta|^2\,\pi(\d\theta).
$
A direct calculation establishes that $\frac{d}{dt} \sqrt{N(\pi_t)} \leq \left|\frac{d}{dt}\,\Risk(\pi_t)\right|^{1/2}$, which implies the following.

\begin{lemma}\label{lemma sublinear growth}\cite[Lemma 3.3]{relutraining}
If $\pi_t$ evolves by the Wasserstein-gradient flow of $\Risk$, then
$
\lim_{t\to \infty} \frac{N(\pi_t)}t = 0.
$
\end{lemma}

\begin{remark}\label{remark decay rate}
If $\Risk(\pi_t)$ is a priori known to decrease at a specific rate, a stronger result holds. Under the fairly restrictive assumption that
\[
\Risk(\pi_t) - \Risk(\pi_{t+1}) \leq C\,t^{-(1+\alpha)},\qquad\text{the estimate}\quad \sqrt{N(\pi_t)} \leq \begin{cases} C\big(1 + t^{\frac{1-\alpha}2}\big) &\alpha>1\\ C\,\log(t+2) &\alpha = 1\end{cases}.
\]
holds.
In particular, if $\Risk(\pi_t) \sim \frac1t$ decays like in the convex case, the most natural decay assumption on the derivative is
$\Risk(\pi_t) - \Risk(\pi_{t+1})  \approx \left|\frac{d}{dt}\,\Risk(\pi_t)\right| \sim t^{-2}
$
which corresponds to $\alpha=1$. Thus, in this case we expect the second moments of $\pi_t$ to blow up at most logarithmically, which agrees with the results of \cite{berlyand2018convergence}.
\end{remark}

\subsection{Slow Approximation Results in High Dimension}

In this section, we recall a result from high-dimensional approximation theory. An infinitely wide two-layer network is a function
\[
f_\pi(x) = \int_{\R^{d+2}} a\,\sigma(w^Tx+b)\,\pi(\d a\otimes \d w \otimes \d b).
\]
The choice of the parameter distribution $\pi$ for $f = f_\pi$ is non-unique since $f_\pi =0$ for all measures $\pi$ which are invariant under the coordinate reflection $T(a,w,b) = (-a, w, b)$. For ReLU activation, further non-uniqueness stems from the fact that
\[
0 = x+1 - x -1 = \sigma(x+1) - \sigma\big(- (x+1)\big) - \sigma(x) + \sigma(-x) - \sigma(1).
\]
The path-norm or Barron norm of a function $f$ is the norm which measures the amount of distortion done to an input along any path which information takes through the network. Due to the non-uniqueness, it is defined as an infimum
\[
\|f\|_{\B} := \inf\left\{ \int |a|\,\big[|w|_{\ell^1} + |b|\big] \,\pi(\d a\otimes \d w \otimes \d b)\:\bigg|\: \pi\in \mathcal P_2 \text{ s.t. }f = f_\pi\right\}.
\]
The equality $f=f_\pi$ is understood in the $\P$-almost everywhere sense for the data distribution $\P$. A more thorough introduction can be found in \cite{weinan2019lei,E:2018ab} or \cite{bach2017breaking}, where a special instance of the same space is referred to as $\F_1$.
Every ReLU-Barron function $f$ is $\|f\|_\B$-Lipschitz. In high dimensions, the opposite is far from true.

\begin{theorem}\cite[Corollary 3.4]{approximationarticle}\label{theorem l2 slow}
Let $d\in\N$. There exists $\phi:[0,1]^d \to [0,1]$ such that 
\[
|\phi(x)-\phi(y)|\leq |x-y|\quad\forall\ x, y\in [0,1]^d\qquad\text{and}\qquad\limsup_{t\to \infty} \left[t^\gamma \, \inf_{\|\psi\|_\B \leq t} \|\phi - \psi \|_{L^2([0,1]^d)}\right] = \infty
\]
for all $\gamma > \frac{2}{d-2}$.
\end{theorem}

This means that
\[
\inf_{\|\psi\|_\B \leq t_k} \|\phi - \psi \|_{L^2([0,1]^d)} \geq t_k^{-\gamma} \quad \text{for all } \gamma> \frac{2}{d-2}
\]
and a sequence of scales $t_k\to \infty$, i.e.\ in high dimension there are Lipschitz functions which are poorly approximated by Barron functions of low norm.
The proof of Theorem is built on the observation that Monte-Carlo integration converges uniformly on Lipschitz-functions and Barron functions with very different rates, suggesting a scale separation.

\section{A Dynamic Curse of Dimensionality}\label{section synthesis}

\begin{proof}[Proof of Theorem \ref{main theorem 2}]
The path-norm of a two-layer neural network is
\begin{align*}
\|f\|_{\B} &= \inf\left\{ \int |a|\,\big[|w|_{\ell^1} + |b|\big] \,\pi(\d a\otimes \d w \otimes \d b)\:\bigg|\: \pi\in \mathcal P_2 \text{ s.t. }f = f_\pi\right\}\\
	&\leq c_d \int |a|^2 + |w|_{\ell^2}^2 + |b|^2\,\bar\pi(\d a\otimes \d w \otimes \d b)
	\qquad= c_d\, N(\bar\pi)
\end{align*}
for {\em any} $\bar\pi$ such that $f_{\bar\pi} = f$. The dimensional constant $c_d= 6 + 4\sqrt{d}$ arises as we apply Young's inequality and invoke the equivalence of the Euclidean norm and the $\ell^1$-norm on $\R^d$. The result now follows from Lemma \ref{lemma sublinear growth} and Theorem \ref{theorem l2 slow}.
\end{proof}

\begin{remark}
The result can be improved under additional assumptions. Like in Remark \ref{remark decay rate}, we assume that the difference quotients of risk satisfy
$
\Risk(\pi_t) - \Risk(\pi_{t+1}) \leq C\,t^{-(\alpha+1)}
$ 
for $\alpha>1$. Then
$
\Risk(\pi_t) = \lim_{s\to\infty}\Risk(\pi_s) + O(t^{-\alpha})\qquad\text{and}\quad 
N(\pi_t) \leq C\,t^{1-\alpha}
$
grows noticeably slower than linearly. If $f^*$ is such that
$
\Risk(\pi_t) \geq C\,N(\pi_t)^{-\beta},
$
and $\Risk(\pi_t)$ decays to zero, we find that
$
\alpha \leq \beta\,(1-\alpha)$, so $\alpha \leq \frac\beta{1+\beta} <\beta$.
\end{remark}

\section{Numerical Results}\label{section simulation}
%
%

\begin{figure}
\includegraphics[width = 0.9\textwidth]{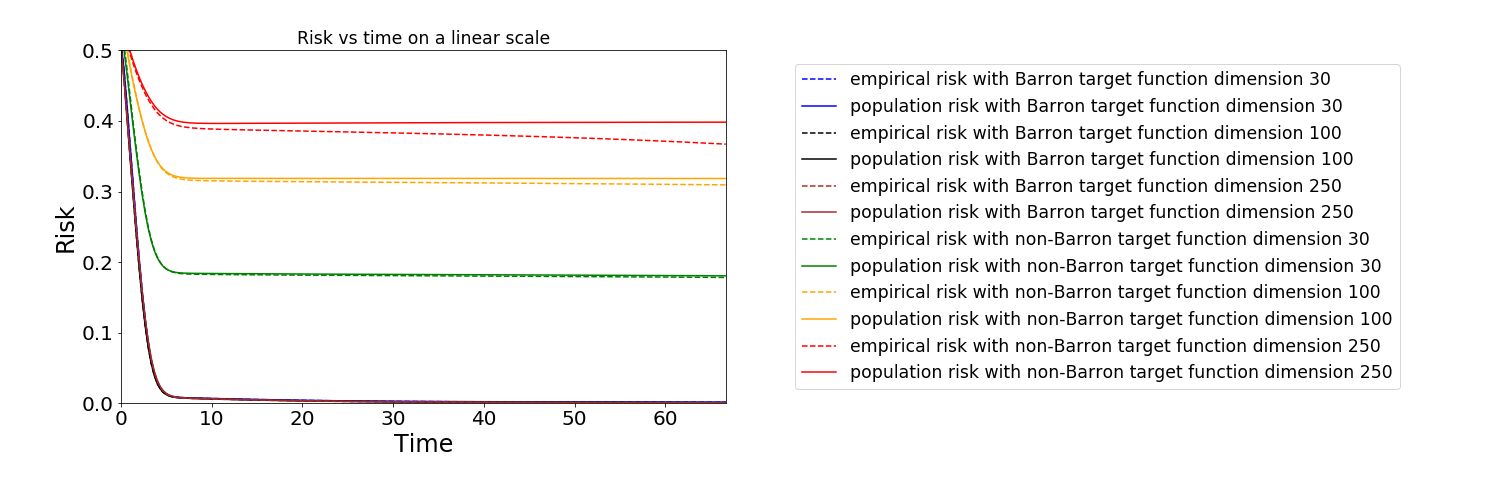}
\caption{\label{figure risk log}A shallow neural network is trained to approximate Barron and non-Barron target functions in moderately high dimension. Plots for Barron target correspond to the colors starting with `b' (blue, black, brown). Their risk decays so similarly across different dimensions that the plots are virtually indistinguishable. For non-Barron target functions, the decay of risk becomes noticeably slower in higher dimensions. Both empirical and population risk are monotone decreasing in all simulations.}
\end{figure}

\begin{figure}
\includegraphics[width = 0.9\textwidth]{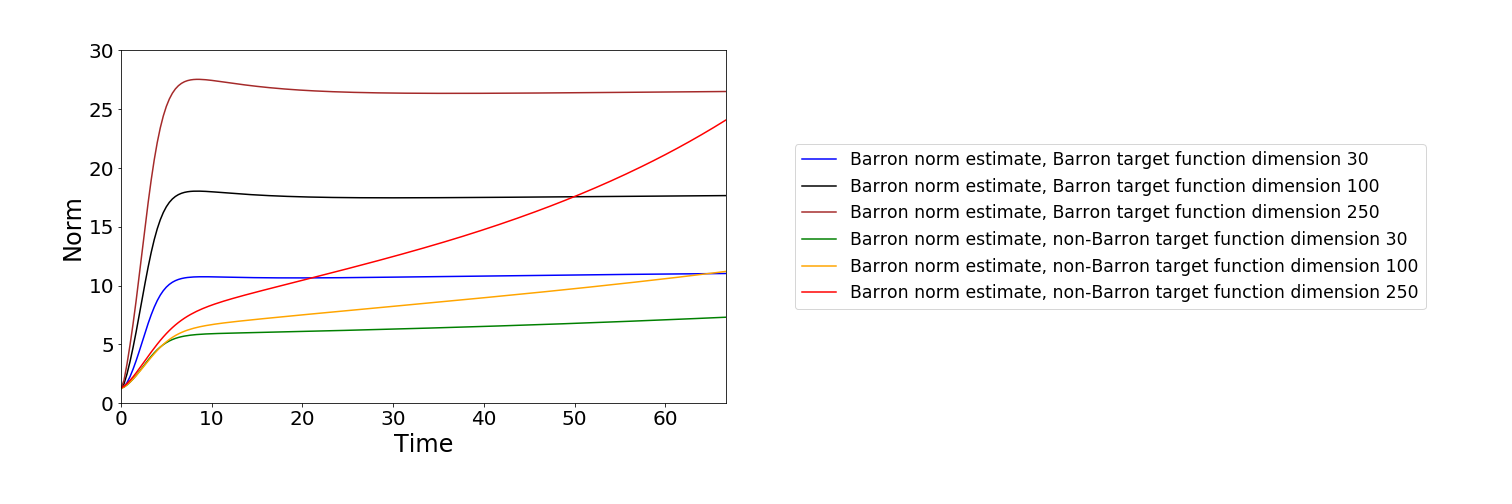}
\caption{\label{figure norm}Network parameters are initialized such that the Barron norm at time $t=0$ is comparable for different dimensions. The Barron norm increases slowly for non-Barron target functions, which may seem counter-intuitive. Recall however that $\frac{d}{dt} \sqrt{N(\pi_t)} \leq \left|\frac{d}{dt}\,\Risk(\pi_t)\right|^{1/2}$, i.e.\ the second moments of the parameter distribution (which bound the Barron norm) can only increase significantly if large amounts of risk are dissipated. Since the decay of risk is slow, also the growth of the Barron norm is slow. However, it is not expected to level off for large times like it does in the Barron case.}
\end{figure}

For $\Theta = \{(a_i, w_i, b_i)\}_{i=1}^m$, we consider the associated two-layer network with ReLU activation
\[
f_\Theta(x) = \frac1m\sum_{i=1}^m a_i\,\sigma(w_i^Tx+b_i) = \frac1m \sum_{i=1}^m a_i\,\big(w_i^Tx+b_i\big)_+.
\] 
As risk functional we choose 
\[
\Risk(\Theta) = \frac12 \fint_{[-1,1]^d} \big|f_\Theta(x) - f^*(x)\big|^2\,\dx \:\:= 2^{-(d+1)}\,\int_{[-1,1]^d} \big|f_\Theta(x) - f^*(x)\big|^2\,\dx.
\]
The target function $f^*$ in our simulations is either
\[
f^*(x) = \sqrt{\frac 32}\big[\|x-a\|_{\ell^2} - \|x + a\|_{\ell^2}\big], \qquad a_i = \frac {2i}d -1
\]
as an example of a Barron function (which can be represented with $a\equiv 1$, $b\equiv 0$ and $w$ distributed uniformly on a sphere of radius $\sim d^{1/2}$, but not with finitely many neurons)  or 
\[
f^*(x) = \sqrt{\frac d\pi}\,\left[\max_{1\leq i \leq d} (x_i - a_i) - \max_{1\leq i\leq d} (-x_i-a_i)\right]
\]
as an example of a Lipschitz continuous, non-Barron target function. In both cases, we have
\[
\int_{[-1,1]^d} f^*(x) \dx = 0, \qquad \|f^*\|_{L^2\big([0,1]^d\big)} \approx 1, \qquad [f^*]_{\mathrm{Lip}} \leq \sqrt{6d}.
\]
For a proof that $f^*$ is not a Barron function, see \cite{barron_new}. In the first case, also the Barron norm of $f^*$ scales as $\sqrt d$. The offset from the origin is used to avoid spurious effects since the initial parameter distribution is symmetric around the origin. 
In simulations, we considered moderately wide networks with $m=1,500$ neurons. The parameters were initialized iid according to Gaussians with expectation $0$ and variance $1$ for $a_i$, $\frac{2}{d+1}I$ for $w_i$, and as constants $b_i = \frac1{2(d+1)}$.
They were optimized by (non-stochastic) gradient descent for an empirical risk functional
\[
\Risk_n(\Theta) = \frac12 \sum_{j=1}^n \big(f_\Theta(x_j) - f^*(x_j)\big)^2
\]
with $n=20,000$ independent samples $x_j\sim U\big([-1,1]^d\big)$. Population risk was approximated by an empirical risk functional evaluated on $N = 100,000$ independent samples. On the data samples, the mean and variance of the target functions were estimated in the range $[-0.013, 0.013]$ and $[1, 1.09]$ respectively for all simulations.

%
%

In Figure \ref{figure risk log} we see that both empirical and population risk decay very similarly for Barron target functions in any dimension, while the decay of risk becomes significantly slower in high dimension for target functions which are not Barron. The empirical decay rate 
\[
\gamma(t) := - \frac{\log(\Risk(\Theta_t))}{\log t}
\qquad\big(\text{which satisfies }
\Risk(\Theta_t) = t^{-\gamma(t)}\big)
\]
becomes smaller for fixed positive time and non-Barron target functions as $d\to \infty$, see Figure \ref{figure risk quotient}. 

Training appears to proceed in two regimes for Barron target functions: An initial phase in which both the Barron norm and risk change rapidly, and a longer phase in which the risk decays gradually and the Barron norm remains roughly constant. In the initial `radial' phase, the vector $(a,w,b)$ is subject to a strong radial force driving the parameters towards the origin or away from the origin exponentially fast. Since $\sigma =$ ReLU is positively $1$-homogeneous, we observe that
\begin{align*}
\frac{(a_i, w_i, b_i)}{\|(a_i, w_i, b_i)\|} \cdot \nabla_{(a_i, w_i, b_i)} \Risk(\Theta) &= \fint_{[-1,1]^d} \big(f_\Theta - f^*(x)\big)\, \frac{(a_i, w_i, b_i)}{\|(a_i, w_i, b_i)\|} \cdot \nabla_{(a_i, w_i, b_i)} f_\Theta(x)\dx\\
	&= \frac1m \fint_{[-1,1]^d} \big(f_\Theta - f^*(x)\big)\,\frac{a\,\sigma(w_i^Tx+b_i)}{\|(a,w,b)\|}\,\dx
\end{align*}
with a positively one-homogeneous right hand side. Thus while $f_\Theta$ is close to its initialization ($\approx 0$ due to symmetry in $a$), the vector $(a,w,b)$ moves towards the origin/away from the origin at an exponential rate, depending on the alignment of $a\sigma(w^T\cdot +b)$ with $f^*$. The exponential growth ceases as $f_\Theta$ becomes sufficiently close to $f^*$ (in the $L^2$-weak topology). 

After the initial strengthening of neurons which are  generally aligned with the target function, we reach a more stable state. In the following `angular' phase, the Barron norm remains constant and directional adjustments to parameters dominate over radial adjustments. Using Figures \ref{figure risk log} and \ref{figure norm}, we can easily spot the transition between the two training regimes at time $t\approx 7$.


\begin{figure}
\includegraphics[width = 0.9\textwidth]{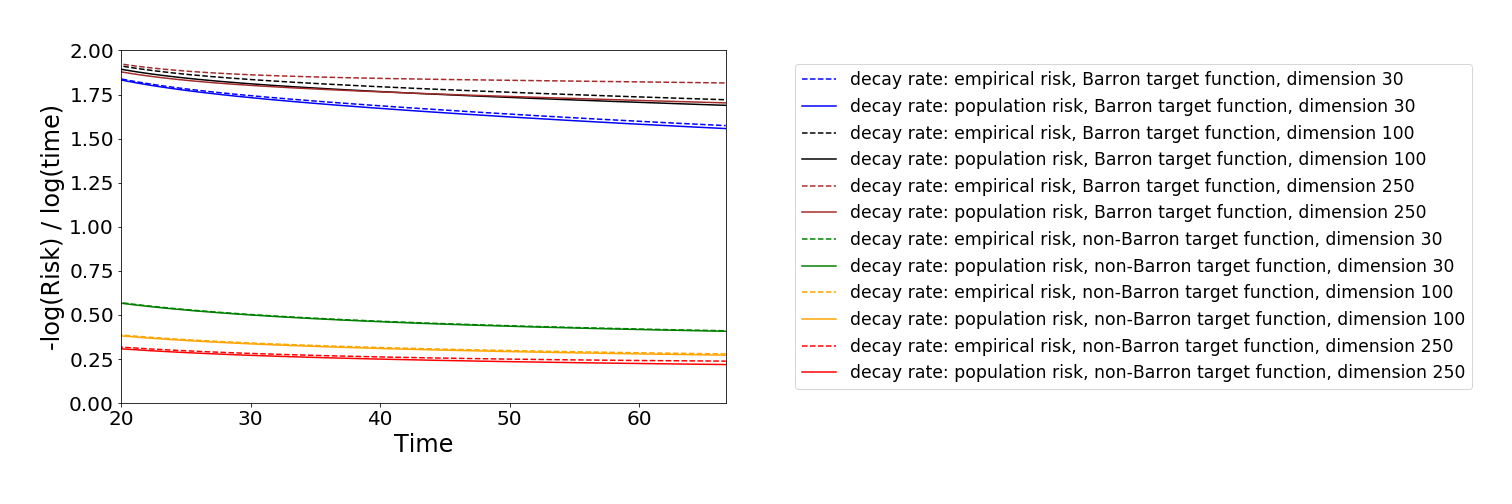}
\caption{\label{figure risk quotient} The empirical decay rate $\gamma(t):= - {\log(\Risk(\Theta_t)}\,/\,{\log t}$ is comparable (potentially even increasing in dimension) for Barron target functions, but gets smaller for non-Barron target functions as the dimension increases. The sample size for empirical risk is sufficiently large that empirical risk and population risk remain similar throughout the evolution.}
\end{figure}


The gap between empirical risk and population risk increases in high dimensions. When training the same networks on the same problems for empirical risk with only 4,000 data points, the results are very similar in dimension 30, but the empirical risk decays very quickly in dimension 250 while the population risk increases rather than decrease when considering a non-Barron target function. This is to be expected since a) the Wasserstein distance between Lebesgue measure and empirical measure increases and b) the number of trainable parameters $m(d+2)$ increases with $d$, making it easier to fit $n$ point values. The risk decays approximately like $t^{-\gamma}$ for Barron target functions with $\gamma\geq 1.5$ (faster in higher dimension). This is faster than expected for generic convex target functions.

\section{Discussion}\label{section discussion}

In this article, we have shown that in the mean field regime, training a two-layer neural network on empirical or population risk may not decrease population risk faster than $t^{-\gamma}$ for $\gamma>4/(d-2)$ when
 the data distribution is truly $d$-dimensional, 
 we consider $L^2$-loss, and
 the target function is merely Lipschitz-continuous, but not in Barron space.
The key ingredient of the result is the slow growth of path norms during gradient flow training and the observation that a Lipschitz function $\phi$ exists which is badly approximable in high dimension.

It is straight-forward to extend the main result to general least-squares minimization.
All statements remain true if instead of `risk decays to zero' we substitute `risk decays to minimum Bayes risk'.

\subsection{Interpretation}

The curse of dimensionality phenomenon occurs when the target function is not in Barron space, i.e.\ a minimizer does not exist. In this situation, even gradient flows of smooth convex functions in one dimension may be slow. The gradient flow ODE
\[
\begin{pde}\dot x_\alpha(t) &= - F_\alpha'(x_\alpha(t)) &t>0\\
	x_\alpha(t) &=1 &t=0\end{pde}
\quad\text{of} \quad F_\alpha:(0,\infty)\to\R,\quad F_\alpha(x) = x^{-\alpha}
\]
is solved by
$
x_\alpha(t)= \big(1+\alpha(\alpha+2)t\big)^\frac1{\alpha+2}.
$
The energy decays as
$
F_\alpha(x_\alpha(t)) \sim t^{-\frac\alpha{\alpha+2}}.
$
If $\alpha \ll 1$, the energy decay is extremely slow. Thus, it should be expected that curse of dimensionality phenomena can occur whenever the risk functional does not have a minimizer in the function space associated with the neural network model under consideration. The numerical evidence of Section \ref{section simulation} suggests that the slow decay phenomenon is visible also in empirical risk if the training sample is large enough (depending on the dimension).

\subsection{Implications for Machine Learning Theory}

Understanding function spaces associated to neural network architectures is of great practical importance. When a minimization problem does not admit a solution in a given function space, gradient descent training may be very slow in high dimension. Unlike the theory of function spaces typically used in low-dimensional problems of elasticity theory, fluid mechanics etc, no comprehensive theory of Banach spaces of neural networks is available except for very special cases \cite{E:2019aa,weinan2019lei}. 
In the light of our result, a convergence proof for mean field gradient descent training of two-layer neural networks must satisfy one of two criteria: It must assume the existence of a minimizer, or it must allow for slow convergence rates in high dimension. 

%
%


\begin{thebibliography}{EMW19b}

\bibitem[AGS08]{ambrosio2008gradient}
Luigi Ambrosio, Nicola Gigli, and Giuseppe Savar{\'e}.
\newblock {\em Gradient flows: in metric spaces and in the space of probability
  measures}.
\newblock Springer Science \& Business Media, 2008.

\bibitem[AKSG19]{arbel2019maximum}
Michael Arbel, Anna Korba, Adil Salim, and Arthur Gretton.
\newblock Maximum mean discrepancy gradient flow.
\newblock In {\em Advances in Neural Information Processing Systems}, pages
  6481--6491, 2019.

\bibitem[AOY19]{araujo2019mean}
Dyego Ara{\'u}jo, Roberto~I Oliveira, and Daniel Yukimura.
\newblock A mean-field limit for certain deep neural networks.
\newblock {\em arXiv:1906.00193 [math.ST]}, 2019.

\bibitem[AS18]{abbe2018provable}
Emmanuel Abbe and Colin Sandon.
\newblock Provable limitations of deep learning.
\newblock {\em arXiv:1812.06369 [cs.LG]}, 2018.

\bibitem[Bac17]{bach2017breaking}
Francis Bach.
\newblock Breaking the curse of dimensionality with convex neural networks.
\newblock {\em The Journal of Machine Learning Research}, 18(1):629--681, 2017.

\bibitem[Bar93]{barron1993universal}
Andrew~R Barron.
\newblock Universal approximation bounds for superpositions of a sigmoidal
  function.
\newblock {\em IEEE Transactions on Information theory}, 39(3):930--945, 1993.

\bibitem[BJ18]{berlyand2018convergence}
Leonid Berlyand and Pierre-Emmanuel Jabin.
\newblock On the convergence of formally diverging neural net-based
  classifiers.
\newblock {\em Comptes Rendus Mathematique}, 356(4):395--405, 2018.

\bibitem[CB18a]{chizat2018note}
Lenaic Chizat and Francis Bach.
\newblock A note on lazy training in supervised differentiable programming.
\newblock {\em arXiv:1812.07956 [math.OC]}, 2018.

\bibitem[CB18b]{chizat2018global}
Lenaic Chizat and Francis Bach.
\newblock On the global convergence of gradient descent for over-parameterized
  models using optimal transport.
\newblock In {\em Advances in neural information processing systems}, pages
  3036--3046, 2018.

\bibitem[CB20]{Chizat:2020aa}
Lenaic Chizat and Francis Bach.
\newblock Implicit bias of gradient descent for wide two-layer neural networks
  trained with the logistic loss.
\newblock {\em arxiv:2002.04486 [math.OC]}, 2020.

\bibitem[Cyb89]{cybenko1989approximation}
George Cybenko.
\newblock Approximation by superpositions of a sigmoidal function.
\newblock {\em Mathematics of control, signals and systems}, 2(4):303--314,
  1989.

\bibitem[DZPS18]{du2018gradient}
Simon~S Du, Xiyu Zhai, Barnabas Poczos, and Aarti Singh.
\newblock Gradient descent provably optimizes over-parameterized neural
  networks.
\newblock {\em arXiv:1810.02054 [cs.LG]}, 2018.

\bibitem[EMW18]{E:2018ab}
Weinan E, Chao Ma, and Lei Wu.
\newblock A priori estimates of the population risk for two-layer neural
  networks.
\newblock {\em Comm. Math. Sci.}, 17(5):1407 -- 1425 (2019), arxiv:1810.06397
  [cs.LG] (2018).

\bibitem[EMW19a]{weinan2019lei}
Weinan E, Chao Ma, and Lei Wu.
\newblock Barron spaces and the compositional function spaces for neural
  network models.
\newblock {\em arXiv:1906.08039 [cs.LG]}, 2019.

\bibitem[EMW19b]{E:2019aa}
Weinan E, Chao Ma, and Lei Wu.
\newblock Machine learning from a continuous viewpoint.
\newblock {\em arxiv:1912.12777 [math.NA]}, 2019.

\bibitem[EMW19c]{weinan2019comparative}
Weinan E, Chao Ma, and Lei Wu.
\newblock A comparative analysis of optimization and generalization properties
  of two-layer neural network and random feature models under gradient descent
  dynamics.
\newblock {\em Sci. China Math.}, https://doi.org/10.1007/s11425-019-1628-5,
  arXiv:1904.04326 [cs.LG] (2019).

\bibitem[EW20a]{barron_new}
Weinan E and Stephan Wojtowytsch.
\newblock Barron functions and their representation.
\newblock {\em In preparation}, 2020.

\bibitem[EW20b]{approximationarticle}
Weinan E and Stephan Wojtowytsch.
\newblock Kolmogorov width decay and poor approximators in machine learning:
  Shallow neural networks, random feature models and neural tangent kernels.
\newblock In preparation, 2020.

\bibitem[HLLL19]{hu2019diffusion}
Wenqing Hu, Chris~Junchi Li, Lei Li, and Jian-Guo Liu.
\newblock On the diffusion approximation of nonconvex stochastic gradient
  descent.
\newblock {\em Annals of Mathematical Sciences and Applications}, 4(1):3--32,
  2019.

\bibitem[Hor91]{hornik1991approximation}
Kurt Hornik.
\newblock Approximation capabilities of multilayer feedforward networks.
\newblock {\em Neural networks}, 4(2):251--257, 1991.

\bibitem[HRSS19]{hu2019mean}
Kaitong Hu, Zhenjie Ren, David Siska, and Lukasz Szpruch.
\newblock Mean-field {L}angevin dynamics and energy landscape of neural
  networks.
\newblock {\em arXiv:1905.07769 [math.PR]}, 2019.

\bibitem[LTE15]{li2015dynamics}
Qianxiao Li, Cheng Tai, and Weinan E.
\newblock Dynamics of stochastic gradient algorithms.
\newblock {\em arXiv:1511.06251 [cs.LG]}, 2015.

\bibitem[MMN18]{mei2018mean}
Song Mei, Andrea Montanari, and Phan-Minh Nguyen.
\newblock A mean field view of the landscape of two-layer neural networks.
\newblock {\em Proceedings of the National Academy of Sciences},
  115(33):E7665--E7671, 2018.

\bibitem[Ngu19]{nguyen2019mean}
Phan-Minh Nguyen.
\newblock Mean field limit of the learning dynamics of multilayer neural
  networks.
\newblock {\em arXiv:1902.02880 [cs.LG]}, 2019.

\bibitem[NP20]{nguyen2020rigorous}
Phan-Minh Nguyen and Huy~Tuan Pham.
\newblock A rigorous framework for the mean field limit of multilayer neural
  networks.
\newblock {\em arXiv:2001.11443 [cs.LG]}, 2020.

\bibitem[Raz18]{raz2018fast}
Ran Raz.
\newblock Fast learning requires good memory: A time-space lower bound for
  parity learning.
\newblock {\em Journal of the ACM (JACM)}, 66(1):1--18, 2018.

\bibitem[RVE18]{rotskoff2018neural}
Grant~M Rotskoff and Eric Vanden-Eijnden.
\newblock Neural networks as interacting particle systems: Asymptotic convexity
  of the loss landscape and universal scaling of the approximation error.
\newblock {\em arXiv:1805.00915 [stat.ML]}, 2018.

\bibitem[San15]{santambrogio2015optimal}
Filippo Santambrogio.
\newblock Optimal transport for applied mathematicians.
\newblock {\em Birkh{\"a}user, NY}, 55(58-63):94, 2015.

\bibitem[Sha18]{shamir2018distribution}
Ohad Shamir.
\newblock Distribution-specific hardness of learning neural networks.
\newblock {\em The Journal of Machine Learning Research}, 19(1):1135--1163,
  2018.

\bibitem[SS19]{sirignano2019mean}
Justin Sirignano and Konstantinos Spiliopoulos.
\newblock Mean field analysis of deep neural networks.
\newblock {\em arXiv:1903.04440 [math.PR]}, 2019.

\bibitem[SS20]{sirignano2018mean}
Justin Sirignano and Konstantinos Spiliopoulos.
\newblock Mean field analysis of neural networks: A law of large numbers.
\newblock {\em SIAM J. Appl. Math}, 80(2):725--752, 2020.

\bibitem[Vil08]{villani2008optimal}
C{\'e}dric Villani.
\newblock {\em Optimal transport: old and new}, volume 338.
\newblock Springer Science \& Business Media, 2008.

\bibitem[Woj20]{relutraining}
Stephan Wojtowytsch.
\newblock On the global convergence of gradient descent training for two-layer
  {R}elu networks in the mean field regime.
\newblock In preparation, 2020.

\end{thebibliography}

\newpage 

\appendix

\section{Random Features and Shallow Neural Networks}

\begin{figure}
\includegraphics[width = 0.85\textwidth]{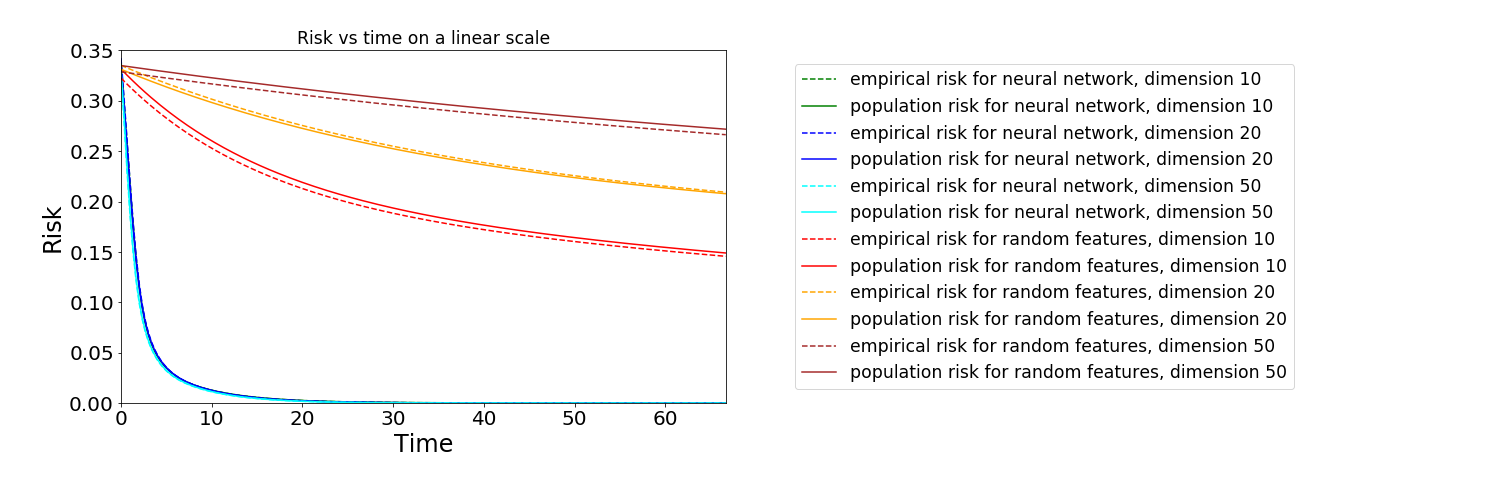}
\caption{\label{figure risk linear rkhs} The risk of two-layer networks when trained to approximate a single neuron activation decays in a largely dimension-independent fashion while the risk of random feature models decays slowly in higher dimension.}
\end{figure}

Lemma \ref{lemma sublinear growth} applies to general models with an underlying linear structure, in particular random feature models. Both two-layer neural networks and random feature models have the form $
f(x) = \frac1m \sum_{i=1}^m a_i\,\sigma(w_i^Tx+b_i),
$
but in random feature models, $w_i, b_i$ is fixed at the (random) initialization. 
An infinitely wide random feature model is described by
\[
f(x) = \int_{\R^{d+1}} a(w,b)\,\sigma(w^Tx+b)\,\pi^0(\d w \otimes \d b)
\]
where $\pi^0$ is a fixed distribution (usually spherical or standard Gaussian) while an infinitely wide two-layer neural network is described by 
\[
f(x) = \int_{\R^{d+1}} a\,\sigma(w^Tx+b)\,\pi (\d a\otimes \d w \otimes \d b).
\]
\cite[Example 4.3]{approximationarticle} establishes a Kolmogorov-width type separation between random feature models and two-layer neural networks of similar form as the separation between two-layer neural networks and Lipschitz functions. Thus a curse of dimensionality also affects the training of infinitely wide random feature models when the target function is a generic Barron function. If $\pi_0$ is a smooth omni-directional distribution and $f^*(x) = \sigma(x_1)$ is a single neuron activation, then $a$ must concentrate a large amount of mass, forcing $\|a\|_{L^2(\pi_0)}$ to blow up. In higher dimension, the blow-up is more pronounced since small balls on the sphere around $e_1$ have faster decaying volume. 

We train a two-layer neural network and a random feature model with gradient descent to approximate the single neuron activation in $L^2([-1,1])$. Both models have width $m=1,500$. Empirical risk is calculated using $10,000$ independent data samples and population risk is approximated using $50,000$ data samples. Both networks are intialized according to a Gaussian distribution as above.

\begin{figure}
\includegraphics[width = 0.85\textwidth]{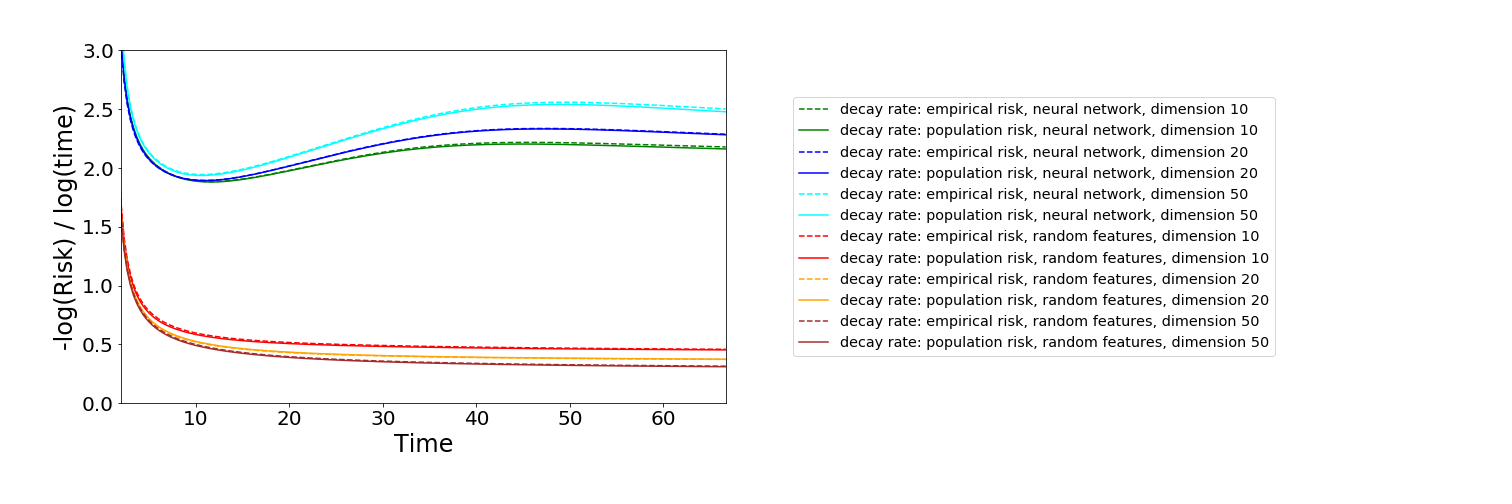}
\caption{\label{figure risk quotient rkhs} Gradient descent optimizes neural networks at a rate of approximately $t^{-2}$ when the target function is a single neuron activation. When a random feature model is used, risk decay is much slower in high dimension. Again, we observe that neural network training appears to work better in high dimension. Empirical and population risk remain close, so we do not attribute this to overfitting.}
\end{figure}

\end{document}